\documentclass[10pt,twocolumn,letterpaper]{article}
\usepackage{cvpr}
\usepackage{times}
\usepackage{epsfig}
\usepackage{graphicx}
\usepackage{amsmath,amssymb,amsthm}
\usepackage{booktabs}
\usepackage{multirow}
\usepackage{bm}
\usepackage{enumitem}
\usepackage{xcolor}
\usepackage{tikz}
\usetikzlibrary{arrows.meta,positioning,calc}
\definecolor{cvprblue}{rgb}{0.21,0.49,0.74}
\definecolor{goodteal}{rgb}{0.08,0.52,0.45}
\definecolor{badred}{rgb}{0.75,0.15,0.15}
\usepackage[pagebackref,breaklinks,colorlinks,allcolors=cvprblue]{hyperref}
\usepackage[capitalize]{cleveref}
\crefname{section}{Sec.}{Secs.}
\Crefname{section}{Section}{Sections}
\crefname{table}{Tab.}{Tabs.}
\Crefname{table}{Table}{Tables}
\crefname{figure}{Fig.}{Figs.}
\Crefname{figure}{Figure}{Figures}
\crefname{proposition}{Prop.}{Props.}
\Crefname{proposition}{Proposition}{Propositions}
\crefname{lemma}{Lemma}{Lemmas}
\Crefname{lemma}{Lemma}{Lemmas}
\def\confName{CVPR}
\def\confYear{2027}
\def\paperID{9692}
\newtheorem{proposition}{Proposition}

\newcommand{\calH}{\mathcal{H}}

\newcommand{\eps}{\varepsilon}
\begin{document}
\title{Representation Before Weighting:\\
Semantic Retention in Frozen 3D Foundation-Model Fusion}
\author{Silas Kwabla Gah \quad Ebenezer Owusu\\
Department of Computer Science, University of Ghana, Legon\\
\texttt{skgah001@st.ug.edu.gh}\quad\texttt{ebeowusu@ug.edu.gh}}
\maketitle
\begin{abstract}

Frozen foundation models are increasingly composed at inference time, yet
a common implementation shortcut collapses each source to a Top-1 label
before heterogeneous evidence can interact. We investigate this
representation bottleneck in frozen two-source 3D semantic fusion,
combining dense RegionPLC predictions with sparse cross-view promptable
segmentation evidence. Source outputs, geometry, scene lists, vocabulary,
and evaluation protocol are held fixed while the number of retained
semantic alternatives is varied. Established information-theoretic
results characterize hard routing as a deterministic information
bottleneck and show that the total-variation error of normalized
Top-$k$ retention equals the discarded probability mass.

On 156 held-out ScanNet200 scenes, under a fixed uniform linear pool,
Top-1 retention reaches 28.47 harmonic-mean (HM) IoU, Top-20 reaches
35.49, and Full retention reaches 34.87: keeping every alternative is
worse than keeping twenty. Paired scene-level bootstrap comparisons show
that Top-20 exceeds Top-10 by $+0.289$ HM (95\% CI
$[+0.093,+0.475]$) and Full by $+0.617$ HM
($[+0.204,+1.086]$). The ordering also depends on the operator:
under geometric pooling it reverses and Full retention is best, at
36.38 HM with a fit-selected exponent.
Top-20 and Full differ on 9.57\% of valid points,
exclusively at locations where the two sources disagree at Top-1.
These changes yield 746,318 net point corrections and are associated
with greater discarded dense-source tail mass. On 50 ScanNet++
validation scenes, distribution-preserving retention also
outperforms Top-1, reaching up to 26.92 HM versus 23.02.
Within ScanNet200 novel classes, Full favors the rarest frequency
tertile, whereas Top-20 favors the most frequent tertile.

Equal-weight geometric pooling reaches 36.16 HM; a global PoE
exponent selected exclusively on the disjoint fit partition
increases held-out performance to 36.38 HM
($+0.223$, 95\% CI $[+0.021,+0.422]$). A calibrated hard selector
that keeps each source's Top-1 and switches by confidence reaches
only 27.51 HM, below naive one-hot averaging at 28.47, so weighting
a hard decision does not recover what collapse removed. A five-scene
GroundingDINO--SAM2.1 source-replacement diagnostic (Appendix)
reproduces the retention benefit
(14.77$\rightarrow$18.75 HM, Top-1 to Full). These findings support
a representation-first design principle for the evaluated
frozen 3D setting: preserve plausible semantic alternatives,
control weak competing tails, and then optimize source weighting.

\end{abstract}

% -------------------------------------------------------------------------
\section{Introduction}
\label{sec:intro}

Open-vocabulary and low-supervision 3D perception increasingly
combines frozen foundation models with complementary inductive
biases. A 3D vision--language prior can provide dense spatial
semantic coverage, while promptable 2D vision models contribute
fine-grained, view-specific object evidence that can be projected
into a 3D scene
\cite{peng2023openscene,takmaz2023openmask3d,yang2024regionplc,luo2024pnpovss}.

Earlier approaches to scarce-supervision semantic segmentation,
including Few-Shot Semantic Segmentation (FSS)
\cite{shaban2017oslsm,wang2019panet,tian2020prior},
Generalized Few-Shot Segmentation (GFSS)
\cite{tian2022generalized,lang2022bam}, and
Open-Vocabulary Semantic Segmentation (OVSS)
\cite{li2022lseg,ghiasi2022openseg,cho2024catseg},
have commonly relied on task-specific parameter optimization
($\Delta\theta\neq\mathbf{0}$). This includes full fine-tuning,
parameter-efficient adaptation such as LoRA
\cite{hu2022lora}, VP-GFSS \cite{hossain2024vp}, and
DIaM \cite{hajimiri2023diam}, as well as test-time adaptation
methods such as Tent \cite{wang2021tent} and CoTTA
\cite{wang2022cotta}. Under scarce supervision, parameter
updates can distort pre-trained representations or compromise
previously learned behavior
\cite{kumar2022finetuning,kirkpatrick2017overcoming}.
This motivates complementary approaches that retain the
pre-trained source representations.

Large-scale vision foundation models span multimodal
contrastive encoders such as CLIP
\cite{radford2021learning} and ALIGN
\cite{jia2021align}, self-supervised visual representations
such as DINOv2 \cite{oquab2023dinov2}, and promptable
segmentation models such as SAM
\cite{kirillov2023sam}, SAM~2 \cite{ravi2024sam2},
and SAM~3 \cite{carion2026sam3}, trained on large-scale
image collections \cite{schuhmann2022laion}. Their
complementary outputs make it practical to compose
pre-trained sources without updating their backbone
parameters \cite{takmaz2023openmask3d,luo2024pnpovss}.
We refer to this operating regime as
\emph{frozen-backbone foundation-model coordination}
($\Delta\theta=\mathbf{0}$): source backbones remain fixed,
while their emitted evidence is related and arbitrated
at inference time. In this regime, system quality depends
not only on which sources are composed, but also on
what representation each source preserves before
their predictions interact.

A common shortcut is to collapse each predictive
distribution $p_m(c\mid x)$ to its most likely class,
$\arg\max_c p_m(c\mid x)$, before combining sources.
Although computationally convenient, this operation
removes runner-up hypotheses before another source can
confirm or contradict them. This is consequential when
heterogeneous sources have complementary errors:
a dense 3D source and a sparse image-based source
need not disagree on the same classes or at the same
spatial locations.

This paper investigates a focused question:

\begin{quote}
\emph{How much of the benefit of frozen multi-source
3D perception is determined by the semantic information
retained before arbitration, rather than by the
complexity of the weighting rule applied afterward?}
\end{quote}

The distinction builds on established ensemble and
information-theoretic principles rather than proposing
new general theory. Classical voting and ensemble
analysis, including the Condorcet Jury Theorem
\cite{condorcet1785essai,boland1989majority,nitzan1982optimal},
correlated voting combiners \cite{ladha1992condorcet},
Bayesian Model Averaging
\cite{kass1995bayes,hoeting1999bayesian}, and
ensemble methods \cite{kittler1998combining,dietterich2000ensemble},
provide context for combining predictions from multiple
sources. The Data Processing Inequality
\cite{cover1999elements} establishes that deterministic
post-processing cannot increase task-relevant information.
We use these established results to motivate a controlled
empirical intervention: source outputs, geometry, scene
lists, vocabulary, and evaluation protocol are fixed,
while the number of semantic alternatives retained
before fusion is varied.

\begin{figure}[t]
\centering
\begin{tikzpicture}[
  font=\scriptsize,
  node distance=0.34cm,
  box/.style={rectangle,rounded corners=3pt,align=center,inner sep=4pt,
              minimum width=0.92\columnwidth},
  src/.style={box,draw=cvprblue!90,fill=cvprblue!7},
  bad/.style={box,draw=badred!85,fill=badred!5},
  good/.style={box,draw=goodteal!90,fill=goodteal!6},
  arr/.style={-{Stealth[length=4pt]},line width=0.9pt}
]
\node[src] (srcs) {\textbf{Frozen heterogeneous sources}\\
Dense RegionPLC posterior + sparse cross-view semantic posterior};
\node[bad,below=of srcs] (hard) {\textbf{Hard routing}\\
$p_m(\cdot\mid x)\rightarrow \arg\max_c p_m(c\mid x)$ before interaction};
\node[good,below=of hard] (retain) {\textbf{Representation-preserving routing}\\
retain Top-$k$ or full posteriors, then combine distributions};
\node[good,below=of retain] (compare) {\textbf{Controlled arbitration}\\
uniform average, calibrated/reliability weighting, max rule, geometric pooling};
\draw[arr] (srcs) -- (hard);
\draw[arr] (srcs.east) -- ++(0.18,0) |- (retain.east);
\draw[arr] (retain) -- (compare);
\end{tikzpicture}
\caption{\textbf{Representation before weighting.}
The central intervention is the semantic representation exposed to downstream
fusion. Hard Top-1 routing irreversibly removes alternatives; Top-$k$ and full
posteriors preserve them. The same frozen source evidence can then be evaluated
under several arbitration operators.}
\label{fig:overview}
\end{figure}

Our main findings are:

\begin{itemize}[leftmargin=1.2em]

  \item \textbf{Retention is beneficial but non-monotonic.}
  Under the matched linear-pooling control on held-out ScanNet200,
  Top-20 reaches 35.49 HM, compared
  with 28.47 for Top-1, 35.20 for Top-10, and 34.87 for
  Full retention; paired bootstrap comparisons resolve both gaps
  (\cref{sec:results}). The ordering reverses under geometric
  pooling, where Full retention is best (\cref{tab:operators}).
  Neither tail alone reproduces the gain: collapsing either
  source to Top-1 while retaining the other's tail stays below
  the both-collapsed control, and only retaining both reaches
  35.49 HM (Appendix).

  \item \textbf{The Top-20/Full difference is localized and
  class-dependent.}
  Their predictions differ on 9.57\% of valid held-out
  points, exclusively where the dense and sparse sources
  disagree at Top-1. Within novel classes, Full
  significantly favors the rarest frequency tertile,
  whereas Top-20 significantly favors the most frequent
  tertile.

  \item \textbf{Semantic retention has the larger observed
  effect than the evaluated weighting refinements.}
  Full uniform averaging reaches 34.87 HM, Top-20 linear
  pooling reaches 35.49, and equal-weight PoE reaches
  36.16. A global PoE exponent selected only on the
  disjoint fit partition reaches 36.38 HM on held-out
  scenes, a paired improvement of $+0.223$ HM
  (95\% CI $[+0.021,+0.422]$).

  \item \textbf{The retention benefit is observed beyond
  the principal source pair, within diagnostic scope.}
  Distribution-preserving retention improves over Top-1
  on ScanNet++, while a five-scene GroundingDINO--SAM2.1
  source-replacement diagnostic reproduces the benefit
  (14.77$\rightarrow$18.75 HM).

\end{itemize}

The resulting claim is scoped: in the evaluated frozen 3D setting,
preserving plausible semantic alternatives before arbitration has
a larger observed effect than the evaluated refinements to source
weighting, without establishing a universally optimal retention
depth or fusion operator.

% -------------------------------------------------------------------------
\section{Related Work}
\label{sec:related}

\paragraph{Adaptation versus frozen coordination.}
Semantic segmentation has progressed from task-specific architectures
\cite{long2015fully,he2017maskrcnn} to Few-Shot Semantic Segmentation (FSS)
\cite{shaban2017oslsm,wang2019panet,tian2020prior} and Generalized Few-Shot
Segmentation (GFSS) \cite{tian2022generalized,lang2022bam}. Most such methods
adapt pretrained representations through parameter updates, including
parameter-efficient adaptation such as LoRA~\cite{hu2022lora}, visual prompting
for GFSS~\cite{hossain2024vp}, DIaM~\cite{hajimiri2023diam}, and test-time
adaptation such as Tent~\cite{wang2021tent} and CoTTA~\cite{wang2022cotta}.
Such updates can introduce representation drift or compromise previously
acquired knowledge~\cite{kumar2022finetuning,kirkpatrick2017overcoming}. We
instead keep every source backbone frozen and intervene only in the
representation of source predictions exposed to downstream fusion.

\paragraph{Open-vocabulary 3D perception and cross-modal fusion.}
Language-aligned representations enable open-vocabulary segmentation in 2D
\cite{li2022lseg,ghiasi2022openseg,cho2024catseg} and 3D
\cite{peng2023openscene,takmaz2023openmask3d,yang2024regionplc}, building on
contrastive vision--language encoders \cite{radford2021learning,jia2021align},
self-supervised visual transformers \cite{oquab2023dinov2}, and promptable
segmenters \cite{kirillov2023sam,ravi2024sam2,carion2026sam3}. Complementary
work combines image-derived and geometric evidence through cross-modal
distillation \cite{yan2022dpass}, shared bird's-eye-view representations
\cite{liu2022bevfusion}, language-aware Gaussian scenes
\cite{li2025semanticsplat}, sparse-voxel features with confidence-aware
registration \cite{wang2026lesv}, and probabilistic cross-modal correspondence
\cite{haitman2026cgenreg}. These methods change how multimodal evidence is
learned, registered, or represented. We hold the emitted evidence and its
geometric lifting fixed and ask how much class-level support should survive
before arbitration.

\paragraph{Multi-view lifting and voting.}
Projecting image-space predictions into 3D and aggregating them across views
is standard in open-vocabulary 3D pipelines
\cite{takmaz2023openmask3d,yang2024regionplc,haitman2026cgenreg}. Those
pipelines fix a voting rule and ask how to label the scene; we hold the
aggregated posterior fixed and ask how many class alternatives must survive
it, which separates the retention decision from the lifting rule.

\paragraph{Semantic retention, reliability, and information limits.}
The closest methods preserve multiple hypotheses or modulate source
influence: DIGS retains Top-$M$ label distributions over Gaussians
\cite{zeng2026digs}, SenseFuse balances image and shape encoders
\cite{han2026sensefuse}, and reliability-aware product-of-experts models expert
availability and trust \cite{palo2026reliability}. Classifier-combination and
correlated-error analyses explain why continuous posteriors can be more
informative than early voting
\cite{kittler1998combining,dietterich2000ensemble,tumer1996error,kass1995bayes,hoeting1999bayesian}.
Calibration methods \cite{guo2017calibration,ovadia2019can,minderer2021revisiting,platt1999probabilistic,kull2019beyond,ding2021local}
and the Brier decomposition \cite{murphy1973vector} show that confidence
quality is distinct from complementarity. Finally, the Data Processing
Inequality \cite{cover1999elements}, the Fannes--Audenaert bound
\cite{audenaert2007sharp}, Popoviciu's inequality
\cite{popoviciu1935equations}, and recent Top-$k$ analyses
\cite{tzachristas2025topk} bound what truncation can preserve. We use these
as context, not as a new bound, and vary retention depth with sources,
geometry, vocabulary, and evaluation all fixed.
\section{Problem Setting}
\label{sec:setting}

Let $\mathcal{C}=\{1,\ldots,C\}$ denote the semantic vocabulary,
and let $Y\in\mathcal{C}$ denote the reference label of a 3D point
$x$. A frozen source $m\in\{1,\ldots,M\}$ emits a categorical
posterior distribution

\begin{equation}
    \mathcal{H}_m(x)
    =
    \big[p_m(c\mid x)\big]_{c=1}^{C}
    \in\Delta^{C-1},
\end{equation}

where $\Delta^{C-1}$ is the probability simplex. The source
backbone parameters remain fixed throughout our experiments.
Some calibration and reliability variants are fitted on a disjoint
labelled partition; consequently, frozen backbones should not be
interpreted as implying label-free evaluation.

We investigate how the amount of class-level information retained
from each source affects subsequent fusion. We compare three
representations.

\paragraph{Hard Top-1 routing.}
The source distribution is reduced to its most probable class:

\begin{equation}
    h_m(x)=\arg\max_{c\in\mathcal{C}}p_m(c\mid x).
\end{equation}

Only the predicted class identifier is exposed to downstream
interaction. All other class hypotheses and their probabilities
are discarded.

\paragraph{Normalized Top-$k$ retention.}
Let $S_k(x)$ denote the indices of the $k$ largest components
of $\mathcal{H}_m(x)$. We retain these components and renormalize
their probabilities:

\begin{equation}
    \tilde p_m^{(k)}(c\mid x)=
    \begin{cases}
    \dfrac{p_m(c\mid x)}
    {\sum_{j\in S_k(x)}p_m(j\mid x)}, & c\in S_k(x),\\[6pt]
    0, & \text{otherwise}.
    \end{cases}
    \label{eq:topk}
\end{equation}

Thus, Top-$k$ preserves the relative probabilities of the retained
classes while removing the remaining hypotheses. The retained
distribution is normalized to sum to one.

\paragraph{Full-distribution retention.}
Full retention preserves the complete source posterior without
truncation or renormalization:

\begin{equation}
    \tilde p_m^{(\mathrm{full})}(c\mid x)=p_m(c\mid x),
    \qquad c\in\mathcal{C}.
\end{equation}

Equivalently, this is the special case $k=C$ of
\cref{eq:topk}, for which no probability mass is discarded.

\paragraph{Representation storage.}
For the 200-category ScanNet200 vocabulary \cite{scannet200}, a
per-point Top-1 stores one uint8 identifier (1 byte), Top-$k$ stores
$k$ identifiers and $k$ fp16 probabilities ($3k$ bytes), and Full
retention stores 200 fp16 probabilities (400 bytes). These counts
describe the serialization in \cref{tab:retention}, exclude container
overhead, and do not affect accuracy, which is computed from cached
probabilities.

% -------------------------------------------------------------------------

\section{What Hard Routing Removes}
\label{sec:theory}

The following propositions provide formal context for the retention
intervention. They are consequences of established information-theoretic
and probability results; the contribution of this paper lies in the
matched empirical characterization that follows.

\begin{proposition}[Hard routing cannot increase task-relevant information]
\label{prop:dpi}

Let $\calH$ denote a predictive posterior, $h=\arg\max_c \calH_c$
its hard decision, and $\Psi(h)$ any deterministic downstream function.
Then

\begin{equation}
    I(\Psi(h);Y)\le I(h;Y)\le I(\calH;Y).
\end{equation}

Equality in the second inequality holds if and only if $h$ is a
sufficient statistic of $\calH$ for $Y$, that is,

\[
    Y\perp\!\!\!\perp\calH\mid h.
\]

If the posterior is calibrated,
$P(Y=c\mid\calH=\bm p)=p_c$, and there exists an argmax decision
cell with positive probability within which $\calH$ is not almost
surely constant, then

\[
    I(h;Y)<I(\calH;Y).
\]

\end{proposition}

\paragraph{Proof sketch.}
The Markov chain
$Y\rightarrow\calH\rightarrow h\rightarrow\Psi(h)$ and the Data Processing
Inequality \cite{cover1999elements} give both inequalities, with equality in
the second requiring $Y\perp\!\!\!\perp\calH\mid h$. Under calibration,
$H(Y\mid\calH)=\mathbb{E}[H(\bm p)]$ whereas
$H(Y\mid h)=\mathbb{E}[H(\mathbb{E}[\bm p\mid h])]$; concavity of entropy then
gives $H(Y\mid\calH)\le H(Y\mid h)$, strictly whenever the posterior varies
within an argmax cell of positive probability; the Appendix gives the
complete derivation. This does not imply that every distribution-level
fusion rule beats every hard rule, only that hard routing can discard
target information present in the original posterior.

\begin{proposition}[Exact Top-$k$ total-variation error]
\label{prop:tv}

Let $\bm p\in\Delta^{C-1}$ and let
$\tilde{\bm p}^{(k)}$ be its normalized Top-$k$ truncation
from \cref{eq:topk}. Define the discarded tail mass

\begin{equation}
    R_k=\sum_{j\notin S_k}p_j.
\end{equation}

Then

\begin{equation}
    \mathrm{TV}\left(\bm p,\tilde{\bm p}^{(k)}\right)
    =
    \frac{1}{2}
    \left\|\bm p-\tilde{\bm p}^{(k)}\right\|_1
    =
    R_k.
\end{equation}

Consequently, the entropy distortion satisfies the
Fannes--Audenaert continuity bound

\begin{equation}
\begin{split}
    \left|H(\bm p)-H(\tilde{\bm p}^{(k)})\right|
    \le
    \min\Big\{
        \log C,\;
        R_k\log(C-1)+H_2(R_k)
    \Big\},
\end{split}
\label{eq:audenaert}
\end{equation}

where $H_2$ is binary entropy.

\end{proposition}

\paragraph{Proof sketch.}
Writing $\tilde p^{(k)}_j=p_j/(1-R_k)$ on $S_k$, the retained entries
contribute $\sum_{j\in S_k}|p_j-p_j/(1-R_k)|=R_k$ to the $\ell_1$ difference
and the discarded entries contribute another $R_k$, so
$\|\bm p-\tilde{\bm p}^{(k)}\|_1=2R_k$ and the total-variation identity
follows. The entropy bound is the Fannes--Audenaert inequality
\cite{audenaert2007sharp} evaluated at $\Delta=R_k$. The Appendix gives
the full derivation.

Together, \cref{prop:dpi,prop:tv} formalize the representation
bottleneck: hard routing can discard task-relevant information, while
normalized Top-$k$ retention incurs an exact error equal to its discarded
tail mass. Neither fixes an optimal depth, so we test both empirically.
% -------------------------------------------------------------------------

\section{Distribution-Level Arbitration}
\label{sec:fusion}

After semantic retention, we compare distribution-level operators while
holding the underlying source evidence fixed. This separates two questions:
which class alternatives remain available to fusion, and how the retained
source distributions contribute to the final prediction.

\paragraph{Uniform linear pooling.}
For $M$ available sources, uniform linear pooling averages their retained
class distributions:
\begin{equation}
p_{\mathrm{lin}}(c\mid x)
=
\frac{1}{M}\sum_{m=1}^{M}p_m(c\mid x).
\label{eq:lin}
\end{equation}
This is the unweighted reference operator.

\paragraph{Reliability-weighted linear pooling.}
A general reliability-weighted pool assigns each source a non-negative
contribution:
\begin{equation}
p_{\mathrm{wlin}}(c\mid x)
=
\frac{\sum_m w_m(x)p_m(c\mid x)}
{\sum_m w_m(x)},
\qquad w_m(x)\geq 0,
\label{eq:wlin}
\end{equation}
whenever the total weight is positive. We compare calibration-derived,
global-reliability, and contextual-reliability weighting using the same
frozen source distributions. These variants change source contribution;
they do not restore class alternatives removed by an earlier retention
operation. The comparison is related to classical ambiguity decompositions
\cite{krogh1995neural}, linear combiner error analysis
\cite{tumer1996error}, and Bayesian model averaging
\cite{kass1995bayes,hoeting1999bayesian}.

\paragraph{Contextual-reliability instantiation.}
The contextual variant predicts source correctness from a standardized
feature vector $\widetilde{\bm f}_m(x)$:
\begin{equation}
p_m^{\mathrm{ctx}}(x)
=
\sigma\!\left(
\bm w_m^{\top}\widetilde{\bm f}_m(x)+b_m
\right).
\label{eq:pctx}
\end{equation}
It then maps this probability to a non-negative reliability score:
\begin{equation}
r_m(x)
=
\left[2p_m^{\mathrm{ctx}}(x)-1\right]_+
=
\max\!\left\{0,\,2p_m^{\mathrm{ctx}}(x)-1\right\}.
\label{eq:ctxclip}
\end{equation}
With source availability $a_m(x)\in\{0,1\}$ and calibrated correctness
$\widehat p_m(x)$, the effective contribution used in the linear pool is
\begin{equation}
\beta_m(x)
=
a_m(x)\,r_m(x)\,\widehat p_m(x),
\qquad
w_m(x)=\beta_m(x).
\label{eq:ctxweight}
\end{equation}
The contextual variant therefore combines predicted instance-level
correctness, a non-negative transformation, source availability, and
calibrated correctness. Its diagnostics isolate the effect of this
construction by removing the positive-part clipping and/or the
calibrated-confidence factor while leaving the retained semantic
distributions unchanged.

\paragraph{Geometric/Product-of-Experts pooling.}
We also evaluate a geometric pool, which aggregates source evidence in
log-probability space. For numerical floor $\eps>0$:
\begin{equation}
p_{\mathrm{poe}}(c\mid x)
=
\frac{
\exp\!\left(
\frac{1}{M}\sum_m
\log\max\{p_m(c\mid x),\eps\}
\right)
}{
\sum_{j\in\mathcal C}
\exp\!\left(
\frac{1}{M}\sum_m
\log\max\{p_m(j\mid x),\eps\}
\right)
}.
\label{eq:poe}
\end{equation}
Unlike linear pooling, this operator emphasizes classes that receive
simultaneously strong support across sources. We report it as an operator
diagnostic, not as a newly proposed probabilistic model.

At points without sparse image evidence, every two-source fusion variant
falls back to the dense RegionPLC posterior. Pooling is applied only where
sparse semantic evidence is available. Thus the comparison concerns
arbitration between available source distributions rather than treating
missing sparse evidence as an additional prediction.

For the two-source setting, we additionally evaluate a fit-selected
weighted PoE:
\begin{equation}
\log s_{\alpha}(c)
=
\alpha\log\max\{p_d(c),\eps\}
+
(1-\alpha)\log\max\{p_s(c),\eps\},
\label{eq:wpoe}
\end{equation}
where $p_d$ and $p_s$ denote the dense and sparse source distributions,
respectively. The exponent
$\alpha\in\{0.1,0.2,\ldots,0.9\}$ is selected by HM on the disjoint
156-scene fit partition and then frozen for held-out evaluation.
Equal-weight PoE is the special case $\alpha=0.5$.

The comparison is intentionally representation-first: it asks whether
the semantic alternatives needed by these operators survive long enough
to interact. The operators alter how retained evidence is combined;
they do not alter the upstream source predictions or the retention
intervention itself.
% -------------------------------------------------------------------------

\section{Experimental Protocol}
\label{sec:protocol}

\paragraph{Data and evaluation universe.}
ScanNet200 \cite{scannet200} contributes a 312-scene cache split
deterministically into 156 fit and 156 held-out evaluation scenes.
Following the GFS-VL class protocol \cite{an2025generalized}, we evaluate the
57 classes with more than 100 occurrences, designating the 12 most frequent as
base and the remaining 45 as novel. ScanNet++ \cite{yeshwanth2023scannetpp}
contributes 50 high-fidelity validation scenes with 30 classes above 80
occurrences, 12 base and 18 novel. Every method predicts over the complete
200-class source vocabulary, so all classes remain active competitors during
retention, fusion, and final argmax; ground-truth points outside the mapped
universe are ignored, and a designated class with zero aggregate union is
omitted from that bootstrap replicate's mIoU. Throughout, \emph{novel} means
held out from the benchmark's supervised training split; it does not imply
absence from foundation-model pre-training.

\paragraph{Frozen sources.}
The principal pair is dense RegionPLC \cite{yang2024regionplc} and sparse
promptable evidence from SAM3 \cite{carion2026sam3}, produced in image space
and lifted into 3D using known camera geometry. Both backbones remain frozen;
sparse observations are accumulated by class and normalized per 3D point.

\paragraph{Sparse lifting and calibration.}
For image $I_t$ with camera intrinsics $K_t$ and camera-to-world pose
$T_t^{cw}$, a world-space point $\bm x_i$ is transformed using
$T_t^{wc}=(T_t^{cw})^{-1}$ and projected as
\begin{equation}
\label{eq:projection}
\widetilde{\bm u}_{it} = K_t\,[\,R_t^{wc}\;\;\bm t_t^{wc}\,]\,\tilde{\bm x}_i ,
\end{equation}
where $\tilde{\bm x}_i=[\bm x_i^{\top},1]^{\top}$.
A correspondence is \emph{frustum-valid} when its camera-space depth is
positive and its rounded image coordinate lies inside the image. The primary
cache applies this rule without an additional sensor-depth or z-buffer
visibility test. If the projected pixel falls inside a concept-conditioned 2D
mask, the observation contributes semantic evidence for that class at the 3D
point. Masks, camera parameters, frames, correspondences, and the sparse cache
are fixed before the Top-$k$ intervention, so every retention and arbitration
variant receives identical upstream evidence. Calibration and reliability
estimators, and the weighted-PoE exponent, are fitted only on the disjoint fit
partition and frozen before held-out evaluation; no foundation-model parameter
is updated. The setting is therefore frozen-backbone inference, but
calibration-based and fit-selected variants are not label-free.

\paragraph{Metrics and uncertainty.}
Let $m_B$ and $m_N$ denote base and novel mean intersection over union;
we report their harmonic mean $\mathrm{HM}=2m_Bm_N/(m_B+m_N)$.
Headline differences use paired scene-level bootstrap 95\% confidence
intervals (2,000 replicates, seed 0), resampling the same scene indices for
both methods and re-aggregating class-wise intersections and unions within
each replicate, which preserves scene pairing and the class-balanced
evaluation objective.
% -------------------------------------------------------------------------
\section{Results}
\label{sec:results}

\begin{table*}[t]
\centering
\small
\caption{\textbf{Matched semantic-retention ladder.}
Only the number of retained semantic alternatives changes. Intervals in the table are paired scene-level 95\% intervals for $\Delta$HM
relative to Top-1; direct Top-20 comparisons are reported in the text. ``Bytes/pt'' is a transparent serialization estimate: Top-1 stores one
uint8 class ID; Top-$k$ stores $k$ uint8 IDs and $k$ fp16 probabilities
($3k$ bytes); Full stores 200 fp16 probabilities (400 bytes). It excludes
container/index overhead and is not used to compute accuracy. The Top-1 entry
uses the evaluator's deterministic tie rule; a fully specified arg-max variant
gives 28.18 HM (Appendix).}
\label{tab:retention}
\setlength{\tabcolsep}{4.0pt}
\begin{tabular}{l r rr rr}
\toprule
& \textbf{Bytes/pt}
& \multicolumn{2}{c}{\textbf{ScanNet200 (156 held-out scenes)}}
& \multicolumn{2}{c}{\textbf{ScanNet++ (50 scenes)}}\\
\cmidrule(lr){3-4}\cmidrule(lr){5-6}
\textbf{Representation}
& \textbf{payload}
& \textbf{HM}
& \textbf{$\Delta$HM vs. Top-1 (95\% CI)}
& \textbf{HM}
& \textbf{$\Delta$HM vs. Top-1 (95\% CI)}\\
\midrule
Top-1 hard decision & 1 & 28.47 & --- & 23.02 & ---\\
Top-2 retained      & 6 & 31.83 & $+3.36\;[+2.46,+4.37]$ & 26.66 & $+3.64\;[+1.83,+6.10]$\\
Top-5 retained      & 15 & 34.48 & $+6.00\;[+4.98,+7.13]$ & 26.87 & $+3.85\;[+1.98,+6.36]$\\
Top-10 retained     & 30 & 35.20 & $+6.73\;[+5.62,+8.00]$ & 26.92 & $+3.90\;[+2.03,+6.45]$\\
Top-20 retained     & 60 & 35.49 & $+7.02\;[+5.89,+8.31]$ & 26.92 & $+3.89\;[+2.03,+6.46]$\\
Full distribution   & 400 & 34.87 & $+6.40\;[+5.24,+7.64]$ & 26.50 & $+3.48\;[+1.64,+5.93]$\\
\bottomrule
\end{tabular}
\end{table*}

\Cref{tab:retention} reports the matched semantic-retention ladder on
ScanNet200 and ScanNet++. Because the source outputs and downstream
linear-pooling rule are held fixed, the comparison isolates the effect
of changing the number of retained class alternatives. These values
belong to the frozen two-source protocol on 156 held-out scenes;
trained GFS-PCS systems report on the full 312-scene validation set
under supervised training, so the ladder is a controlled ordering
rather than a leaderboard entry.

On ScanNet200, performance rises from 28.47 HM under Top-1 to 34.48
under Top-5, 35.20 under Top-10, and 35.49 under Top-20. It then
decreases to 34.87 under Full retention. Paired scene-level bootstrap
comparisons show that Top-20 exceeds Top-10 by $+0.289$ HM
(95\% CI $[+0.093,+0.475]$) and Full by $+0.617$ HM
($[+0.204,+1.086]$). The Full-versus-Top-10 difference is not
statistically resolved ($-0.328$, 95\% CI $[-0.913,+0.199]$).
Retention is therefore non-monotonic under this matched rule.

The gain also requires retaining alternatives in both sources. Holding the
pool and the upstream evidence fixed and crossing the two depths
(Appendix), collapsing only the sparse source to Top-1 falls to 21.71 HM
and collapsing only the dense source reaches 27.11 HM, both below the 28.18
obtained when both sources are collapsed; retaining both reaches 35.49 HM.

ScanNet++ independently supports the benefit of preserving
distributions over Top-1. Top-10 and Top-20 both reach 26.92 HM,
compared with 23.02 for Top-1; Full reaches 26.50. The relatively
small differences among Top-5, Top-10, and Top-20 on this benchmark
also caution against interpreting the ScanNet200 optimum as a
universal choice of $k$. Because the sparse source covers novel
categories only, richer fusion trades base for novel accuracy
(base mIoU $32.12\rightarrow27.21$): the ScanNet++ gains are a
generalised-balance trade-off rather than a Pareto improvement
over the dense prior on both class groups.

\paragraph{Where does Top-20 differ from Full?}
A point-level audit over the 21,496,033 valid held-out ScanNet200
points shows that Top-20 and Full produce different predictions at
2,056,707 points (9.57\%). Every changed prediction occurs where the
dense and sparse sources disagree at Top-1. Among these changed
points, 952,235 predictions correct a Full error, while 205,917
replace a correct Full prediction with an error, yielding 746,318
net point corrections. A further 898,555 changed points are
incorrect under both variants.

The mean dense probability mass discarded below Top-20 is 0.571
on changed points versus 0.333 on unchanged points; the sparse
discarded mass is negligible (0.002 versus 0.001). The Top-20/Full
gap therefore tracks the dense source's lower-ranked mass at
source-disagreement points.

The aggregate Top-20 advantage is significant for base classes
($+2.52$ mIoU, 95\% CI $[+1.99,+3.02]$), whereas the aggregate
novel-class difference is not resolved ($-0.17$,
95\% CI $[-0.68,+0.42]$). The novel-class frequency analysis
reveals a more differentiated pattern: Full outperforms Top-20
for the rarest frequency tertile (Top-20 minus Full:
$-1.93$ mIoU, 95\% CI $[-3.06,-0.68]$); the middle tertile
is inconclusive ($+0.76$, 95\% CI $[-0.03,+1.64]$);
and Top-20 outperforms Full for the most frequent tertile
($+0.66$, 95\% CI $[+0.30,+1.02]$).

The discarded tail is therefore not uniformly uninformative: longer
retention preserves rare-class alternatives, while truncation removes
weak competing mass at source-disagreement points.

\paragraph{Robustness of the matched comparison.}
Averaging disagreeing one-hot predictions produces exact ties, so we verify
that the Top-1 control does not depend on the tie convention. Ties arise at
67.9\% of sparse-covered valid points. Resolving them in favour of the dense
source gives 26.53 HM, in favour of the sparse source 20.02 HM, and with the
frozen evaluator's rule 28.47 HM; a fully specified variant that takes each
source's arg-max before averaging gives 28.18 HM (Appendix). Each
alternative therefore widens the gap to Full retention ($+6.69$ to $+14.86$
HM), so the reported advantage is not an artifact of tie ordering. The sparse
weight was also swept over
$w_s\in[0.3,0.7]$: Full retention remains within 34.17--34.87 HM across the
sweep (spread 0.70), whereas the Top-1 control is piecewise constant because
one-hot averaging turns the weight into a hard switch between sources.
Retaining class-level support thus yields an arbitration whose outcome is
stable under the weighting choice and determined instead by the retained
semantic support.

\paragraph{Retention and storage.}
Under the serialization of \cref{tab:retention}, Top-5 costs 15 bytes
per point against 60 for Top-20 and 400 for Full, so the most accurate
retention depth is not the cheapest operating point.

\subsection{Simple Distribution-Level Operators Are Strong}
\begin{figure}[t]
\centering
\setlength{\tabcolsep}{1.4pt}
\begin{tabular}{@{}c@{}c@{}c@{}c@{}}
\includegraphics[width=0.234\columnwidth]{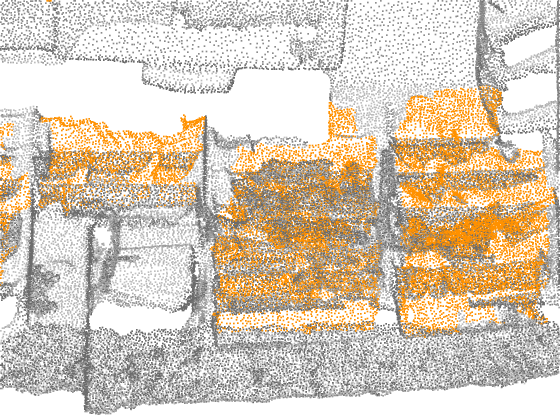} &
\includegraphics[width=0.234\columnwidth]{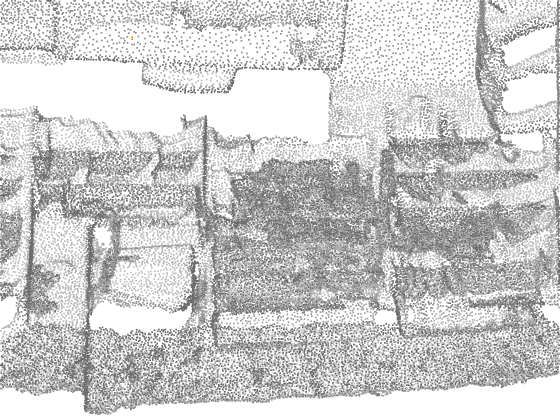} &
\includegraphics[width=0.234\columnwidth]{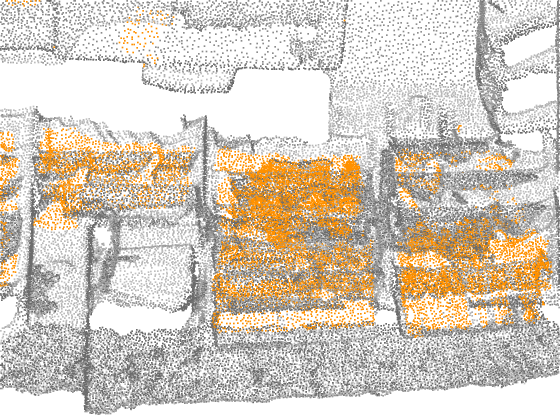} &
\includegraphics[width=0.234\columnwidth]{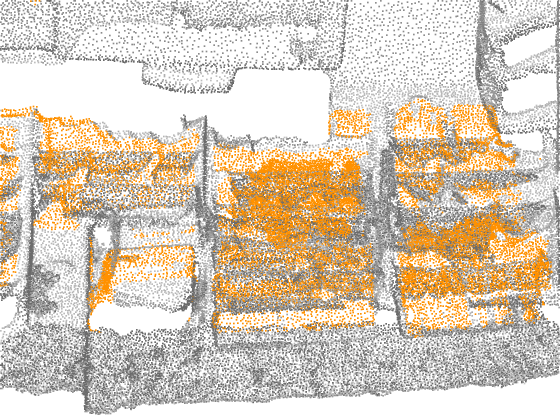}\\[-1pt]
\scriptsize Ground truth & Top-1 & Top-20 + linear & Full + PoE
\end{tabular}
\caption{\textbf{Qualitative effect of retention depth.} Held-out ScanNet200
scene \texttt{scene0030\_01}; only the novel class \emph{book} is coloured and
all other labels are grey. Of its 23{,}723 book points the matched Top-1
control recovers 1, Top-20 with linear pooling $53.3\%$, and Full retention
with fit-weighted PoE ($\alpha=0.4$) $68.9\%$.}
\label{fig:qual}
\end{figure}
\begin{table}[t]
\centering
\small
\caption{\textbf{Routing and arbitration variants on held-out ScanNet200.}
All rows use the same frozen source outputs. Rows distinguish hard routing,
full-distribution weighting variants, and additional retention/operator diagnostics.}
\label{tab:operators}
\setlength{\tabcolsep}{3.3pt}
\begin{tabular}{l rrr}
\toprule
\textbf{Arbitration rule} & \textbf{Base} & \textbf{Novel} & \textbf{HM}\\
\midrule
Dense RegionPLC only & 43.60 & 19.06 & 26.53\\
Hard calibrated selector & 41.84 & 20.49 & 27.51\\
Top-1 retained one-hot & 45.13 & 20.80 & 28.47\\
\midrule
Uniform linear pooling & 44.43 & 28.70 & 34.87\\
Calibration-weighted pooling & 45.44 & 27.82 & 34.51\\
Global-reliability pooling & 45.70 & 27.56 & 34.38\\
Contextual-reliability pooling & 44.06 & 20.31 & 27.80\\
Top-20 + linear pooling & 46.95 & 28.53 & 35.49\\
Top-20 + equal PoE ($\alpha=.5$) & \textbf{47.08} & 28.89 & 35.81\\
Top-20 + fit-weighted PoE ($\alpha=.4$) & 46.59 & 29.18 & 35.89\\
Full + equal PoE ($\alpha=.5$) & 47.07 & 29.35 & 36.16\\
Full + fit-weighted PoE ($\alpha=.4$) & 46.62 & \textbf{29.83} & \textbf{36.38}\\
\bottomrule
\end{tabular}
\end{table}

\cref{tab:operators} separates preserving alternatives from changing
the operator. With Full retention, uniform linear pooling reaches
34.87 HM; the same operator on Top-20 retained distributions reaches
35.49 HM. Calibration-weighted (34.51), global-reliability (34.38),
and contextual-reliability (27.80) pooling all fall below Full
uniform averaging.

The confidence-weighted hard baseline in \cref{tab:operators}
keeps each source's Top-1 label and switches by calibrated
probability of correctness, so no one-hot average or tie rule is
involved; it reaches 27.51 HM. We report it rather than a literal
comparison of raw maxima because the dense maximum is a 200-class
softmax while the sparse maximum is a cross-view agreement fraction,
so the two are not on a common scale.

The geometric operator provides a stronger result. Equal-weight
PoE reaches 36.16 HM. Selecting the dense-source exponent on
the disjoint fit partition gives $\alpha=0.4$; the resulting
weighted PoE reaches 36.38 HM on held-out scenes. Its paired
improvement over equal-weight PoE is $+0.223$ HM
(95\% CI $[+0.021,+0.422]$). This improvement accompanies
a $-0.454$ change in base mIoU and a $+0.478$ change in novel mIoU.

The two effects are not additive. Under linear pooling Top-20 beats
Full by $+0.62$ HM, but under the geometric operator Top-20 lowers HM
from 36.16 to 35.81, with the same ordering when weighted (35.89
versus 36.38). The retained tail and the log-space operator therefore
act on the same low-support evidence rather than contributing
independent gains, so reducing support does not substitute for the
choice of operator.

For scale, a per-point oracle that always returns the source whose
Top-1 is correct reaches 42.71 HM; it uses point-level labels, so it
only bounds the headroom in this source pair.

Taken together, these comparisons show that whether to retain
alternatives at all matters more than the choice among the evaluated
weighting rules, although at the margin the two interact rather than
compose: weighting can refine available evidence but cannot recover
what fusion never saw.

\paragraph{PoE numerical stability.}
The sparse source contains exact zeros, so \cref{eq:poe} uses a
numerical floor $\eps$; sweeping $\eps\in\{10^{-4},\dots,10^{-12}\}$
changes HM by at most 0.16 points and leaves the reported result
unchanged for $\eps\leq10^{-8}$ (Appendix).

\begin{table}[t]
\centering
\small
\caption{\textbf{Calibration and contextual-weight diagnostics on held-out ScanNet200.} Upper: post-hoc calibration sharply improves ECE and Brier score
for each source. Lower: the contextual-fusion deficit is largely associated
with the non-negative clipping transformation rather than with semantic
retention itself.}
\label{tab:caldiag}
\setlength{\tabcolsep}{5pt}
\begin{minipage}[t]{0.47\textwidth}
\centering
\begin{tabular}{llrr}
\toprule
\textbf{Source} & \textbf{State} & \textbf{ECE (\%)} & \textbf{Brier}\\
\midrule
RegionPLC & Raw & 29.20 & 0.3117\\
RegionPLC & Calibrated & 1.54 & 0.2146\\
SAM3 & Raw & 43.29 & 0.4187\\
SAM3 & Calibrated & 0.69 & 0.2190\\
\bottomrule
\end{tabular}
\end{minipage}
\hfill
\begin{minipage}[t]{0.47\textwidth}
\centering
\begin{tabular}{lr}
\toprule
\textbf{Fusion diagnostic} & \textbf{HM}\\
\midrule
Uniform distribution averaging & 34.87\\
Contextual $\times$ calibrated, clipped & 27.80\\
Contextual $\times$ calibrated, unclipped & 33.25\\
Clipped calibrated confidence only & 27.82\\
Contextual probability only & 33.83\\
\bottomrule
\end{tabular}
\end{minipage}
\end{table}

\subsection{Calibration Improves Probabilities, Not Necessarily Fusion}

\cref{tab:caldiag} shows large gains in marginal probability quality:
ECE falls to 1.54\% (RegionPLC) and 0.69\% (SAM3), with Brier scores
of 0.2146 and 0.2190.

Calibration does not translate into better fusion: calibration-weighted
pooling reaches 34.51 HM, below the 34.87 HM of uniform averaging
(\cref{tab:operators}). Calibration measures confidence--correctness
agreement, whereas fusion also depends on which classes each source
supports and on how their errors relate.

The contextual-weight diagnostics further show that the transformation
from estimated reliability to fusion weight matters. The original
contextual variant, combining contextual correctness, calibrated
confidence, and non-negative clipping, reaches 27.80 HM. Removing the
positive-part clipping raises HM to 33.25, so the deficit tracks the
weight construction rather than semantic retention.

% ------------------------------------------------------------------------
\section{Discussion and Limitations}
\label{sec:discussion}

Across both benchmarks, retaining several class alternatives beats hard Top-1
routing, and on ScanNet200 the benefit is non-monotonic: Top-20 exceeds
Top-10 and Full under the same linear pool. The point audit localizes every
Top-20/Full change to source-disagreement points, and the frequency analysis
shows the discarded tail is not uniformly harmful: Full favors the rarest
novel classes, Top-20 the most frequent. Weighting forms a second, smaller
tier: equal-weight PoE reaches 36.16 HM and a fit-selected exponent 36.38 HM
($+0.223$, 95\% CI $[+0.021,+0.422]$), well below the gain from leaving
Top-1 collapse. Retention is thus the first-order decision and weighting a
refinement.

Calibration separates probability quality from fusion quality: post-hoc
calibration cuts ECE from 29.20\% to 1.54\% and from 43.29\% to 0.69\%, yet
calibration-weighted pooling (34.51 HM) trails uniform averaging (34.87 HM),
and the weight construction itself moves HM by up to 5.45 points. Better
marginal calibration does not imply that a source deserves more influence at
every point.

These findings support a representation-first principle for the evaluated
setting: preserve plausible alternatives before deciding how to combine
sources, without assuming the full low-probability tail helps. The principle
prescribes neither a universal depth nor a universally optimal operator.

The scope is deliberately limited. The principal experiments cover two frozen
semantic sources and two indoor 3D benchmarks, and the source-replacement
diagnostic covers five scenes (Appendix), so the claim is directional rather
than benchmark-scale. The sparse cache uses frustum-valid projection without an
explicit sensor-depth or z-buffer gate; the evidence is fixed across matched
comparisons, but projection and visibility errors may affect absolute
sparse-source quality. Reported CPU timing covers post-extraction arbitration
only, excluding source forward passes (Appendix). Confidence intervals are
per comparison and not corrected for multiplicity; the principal retention,
PoE, and calibration contrasts carry their own intervals, while the
frequency-tertile and cross-dataset contrasts are descriptive. Finally, no
single retention depth is optimal across all evaluated conditions.

\section{Conclusion}
\label{sec:conclusion}

We investigated how semantic information should be retained before
heterogeneous frozen 3D foundation-model predictions are combined. Through
matched experiments on ScanNet200 and ScanNet++, hard Top-1 routing discards
useful semantic alternatives, while retaining the complete distribution is not
uniformly optimal. Preserving plausible alternatives and controlling weak
competing mass therefore precede the choice of fusion weight: in this setting,
representation, rather than the weighting refinement, dominates.

\section*{LLM Usage Statement}
OpenAI ChatGPT assisted interactively with language editing, literature
search, LaTeX preparation, diagnostic experimental design, and analysis code.
All experiments were run by the authors, who reviewed the generated code and
verified all citations, numbers, statements, and conclusions.

\appendix
\section{Complete Proofs}
\label{sec:supp-proofs}

\subsection{Proposition 1: hard routing cannot increase task-relevant information}

Let $\calH$ be a predictive posterior over the vocabulary
$\mathcal{C}=\{1,\dots,C\}$, let $h=\arg\max_c \calH_c$ be its hard decision,
and let $\Psi(h)$ be any deterministic downstream function of $h$. The
variables form the Markov chain
$Y\rightarrow \calH\rightarrow h\rightarrow \Psi(h)$.
Applying the Data Processing Inequality \cite{cover1999elements} to
$\calH\rightarrow h\rightarrow\Psi(h)$ and to
$Y\rightarrow\calH\rightarrow h$ gives
\begin{equation}
    I(\Psi(h);Y)\le I(h;Y)\le I(\calH;Y).
    \label{eq:supp-dpi}
\end{equation}
Equality in the second inequality holds if and only if $h$ is a sufficient
statistic of $\calH$ for $Y$, that is, $Y\perp\!\!\!\perp \calH\mid h$.

We now show that the second inequality is strict under calibration. Assume
$P(Y=c\mid \calH=\bm p)=p_c$ for every $\bm p$ in the simplex. Then
\begin{equation}
    H(Y\mid \calH)=\mathbb{E}\left[H(\bm p)\right],
    \label{eq:supp-condH}
\end{equation}
whereas
\begin{equation}
    H(Y\mid h)
    =
    \mathbb{E}\left[H\left(\mathbb{E}[\bm p\mid h]\right)\right].
    \label{eq:supp-condHh}
\end{equation}
For every decision cell, the conditional expectation
$\mathbb{E}[\bm p\mid h]$ is a convex combination of posterior vectors
within that cell. Since Shannon entropy is concave,
\begin{equation}
    \mathbb{E}\left[H(\bm p)\mid h\right]
    \le
    H\left(\mathbb{E}[\bm p\mid h]\right).
    \label{eq:supp-jensen}
\end{equation}
Taking expectations over $h$ and substituting
\cref{eq:supp-condH,eq:supp-condHh} yields
$H(Y\mid\calH)\le H(Y\mid h)$, hence $I(\calH;Y)\ge I(h;Y)$.

If there exists an argmax decision cell with positive probability within which
$\calH$ is not almost surely constant, then $\bm p$ is non-degenerate on a set
of positive measure inside that cell. Because entropy is strictly concave,
\cref{eq:supp-jensen} is strict on that cell, and therefore
$H(Y\mid\calH)<H(Y\mid h)$ and $I(\calH;Y)>I(h;Y)$.

The proposition concerns the information available about $Y$; it does not
assert that every distribution-level fusion rule outperforms every hard-fusion
rule on a finite evaluation sample. A fusion rule may fail to exploit the
retained information, and finite-sample estimation introduces additional error.
The result establishes only that hard routing can discard task-relevant
information that is present in the original posterior, which is the
representation bottleneck studied empirically in the main paper.

\subsection{Proposition 2: exact Top-$k$ total-variation error}

Let $\bm p\in\Delta^{C-1}$, let $S_k$ index the $k$ largest entries of
$\bm p$, and let $\tilde{\bm p}^{(k)}$ be the normalized Top-$k$ truncation
\begin{equation}
    \tilde p^{(k)}_j=
    \begin{cases}
        p_j/(1-R_k), & j\in S_k,\\
        0, & j\notin S_k,
    \end{cases}
    \qquad
    R_k=\sum_{j\notin S_k}p_j .
    \label{eq:supp-topk}
\end{equation}
The total-variation distance is
$\mathrm{TV}(\bm p,\tilde{\bm p}^{(k)})
=\tfrac12\|\bm p-\tilde{\bm p}^{(k)}\|_1$.
For the retained entries,
\begin{equation}
    \sum_{j\in S_k}
    \left|p_j-\frac{p_j}{1-R_k}\right|
    =
    \sum_{j\in S_k}p_j\left(\frac{1}{1-R_k}-1\right)
    =
    R_k .
    \label{eq:supp-retained}
\end{equation}
The discarded entries contribute
$\sum_{j\notin S_k}p_j=R_k$ to the $\ell_1$ difference. Therefore
\begin{equation}
    \left\|\bm p-\tilde{\bm p}^{(k)}\right\|_1=2R_k,
    \qquad
    \mathrm{TV}\left(\bm p,\tilde{\bm p}^{(k)}\right)=R_k .
    \label{eq:supp-tv}
\end{equation}
The entropy distortion follows from the Fannes--Audenaert continuity bound
\cite{audenaert2007sharp} applied with $\Delta=R_k$:
\begin{equation}
\begin{split}
    \left|H(\bm p)-H(\tilde{\bm p}^{(k)})\right|
    \le
    \min\Big\{
        \log C,\;
        R_k\log(C-1)+H_2(R_k)
    \Big\},
\end{split}
\label{eq:supp-fannes}
\end{equation}
where $H_2$ is the binary entropy function. The identity in
\cref{eq:supp-tv} is exact for every $k$ and every categorical distribution;
it does not depend on how the retained entries are subsequently weighted.

Two consequences are used in the main paper. First, hard Top-1 routing is the
limiting case $k=1$, for which the discarded mass can be large whenever the
posterior is not concentrated on a single class. Second, normalized Top-$k$
retention introduces a probability-space perturbation whose magnitude is
exactly the discarded tail mass. Neither consequence determines the optimal
retention depth, because the discarded mass may contain either useful
alternatives or misleading evidence.

\section{Additional Protocol Details}
\label{sec:supp-protocol}

The main paper states the evaluation universe and the retained class counts.
For completeness, this section records the lifting rule and the supervision
boundary that hold for every reported condition.

For image $I_t$ with intrinsics $K_t$ and camera-to-world pose $T_t^{cw}$,
a world-space point $\bm x_i$ is mapped to image coordinates
\begin{equation}
\widetilde{\bm u}_{it}
=
K_t
\begin{bmatrix}
R_t^{wc} & \bm t_t^{wc}
\end{bmatrix}
\begin{bmatrix}
\bm x_i\\
1
\end{bmatrix},
\qquad
T_t^{wc}=(T_t^{cw})^{-1}.
\label{eq:supp-projection}
\end{equation}
After perspective division, a point--view correspondence is retained when the
camera-space depth is positive and the rounded image coordinate lies inside
the image. The primary sparse-evidence cache uses this frustum-valid rule
without an additional sensor-depth or z-buffer visibility test. If the
projected pixel falls inside a concept-conditioned 2D mask, the observation
contributes evidence for that class at the 3D point. Masks, camera
parameters, the frame set, the correspondences, and the resulting sparse
cache are fixed before the Top-$k$ intervention, so every retention and
arbitration variant receives identical upstream evidence.

Calibration and reliability estimators are fitted only on the disjoint
ScanNet200 fit partition and then frozen. The weighted-PoE exponent is
selected on the same partition and frozen before held-out evaluation. No
foundation-model parameter is updated. The setting is therefore
frozen-backbone inference, but calibration-based and fit-selected variants are
not label-free: they use supervision from the fit partition.

\paragraph{Frozen sources.}
The dense source is RegionPLC with frozen weights. The sparse branch uses the
frozen SAM3 promptable-concept-segmentation interface with the official
benchmark class names, without synonym expansion or prompt ensembling. Camera
intrinsics and poses are those supplied by each benchmark; no additional
calibration, bundle adjustment, or pose refinement is applied. All
foundation-model parameters are fixed at deployment.

\paragraph{Cross-fitted operating-point selection.}
For ScanNet++, the 50 validation scenes are assigned to two disjoint 25-scene
folds by ordering scene identifiers with a deterministic SHA-256 hash, so fold
membership is fixed independently of any segmentation output. Each fold is
scored with the scalar consensus threshold chosen on the opposite fold from
the grid $\{0.10,0.15,0.20,0.25,0.30,0.35,0.40,0.50\}$. Selecting on fold A
gives $\tau=0.40$, which reaches 32.21 HM on held-out fold B; reversing the
roles gives $\tau=0.30$, which reaches 32.80 HM on held-out fold A. Pooling
only the predictions produced while each scene was held out yields 31.81 base
mIoU, 32.20 novel mIoU, and 32.00 HM over all 50 scenes. No scene therefore
contributes to both the selection and the evaluation of its own operating
threshold. The procedure is leakage-free cross-fitted threshold selection, but
it remains label-assisted rather than label-free.

\section{Robustness of the Matched Comparison}
\label{sec:supp-robust}

Two properties of the matched Top-1 control could in principle produce the
reported retention effect rather than the retention itself: the deterministic
rule used to resolve tied one-hot averages, and the scalar weight assigned to
the sparse source. We vary both while holding all upstream evidence, scene
lists, and evaluation code fixed.

\paragraph{Tie resolution.}
Across the 156 held-out scenes, 20,402,026 valid points receive sparse
evidence and 13,859,946 of those have disagreeing dense and sparse Top-1
predictions. Equal averaging of disagreeing one-hot vectors therefore creates
an exact tie at 67.934\% of sparse-covered valid points. \Cref{tab:supp-tie}
resolves these ties three ways.

\begin{table}[t]
\centering
\small
\caption{\textbf{Effect of deterministic tie resolution on the matched Top-1
control.} Full-distribution fusion is unchanged by the tie rule. The fully
specified row takes each source's arg-max before averaging, so no retained
entry depends on a partition order.}
\label{tab:supp-tie}
\setlength{\tabcolsep}{3.6pt}
\begin{tabular}{l rrrr}
\toprule
\textbf{Top-1 policy} & \textbf{Base} & \textbf{Novel} & \textbf{HM} & \textbf{Full$-$Top-1}\\
\midrule
Frozen evaluator   & 45.13 & 20.80 & 28.47 & $+6.40$\\
Fully specified    & 44.86 & 20.55 & 28.18 & $+6.69$\\
Dense-favouring    & 43.60 & 19.06 & 26.53 & $+8.35$\\
Sparse-favouring   & 23.98 & 17.18 & 20.02 & $+14.86$\\
Full distribution  & 44.43 & 28.70 & 34.87 & ---\\
\bottomrule
\end{tabular}
\end{table}

The frozen evaluator selects each source's retained entry with a
partition-based top-$k$ routine, so where several classes share the maximum
count the retained class depends on the partition order rather than on a
stated rule. Taking each source's arg-max explicitly and resolving the
averaging tie by class index changes the matched control by $-0.29$ HM and
leaves every conclusion unchanged. Each fully specified alternative produces
a larger gap than the frozen evaluator's rule ($+6.69$ to $+14.86$ versus
$+6.40$), so the full-distribution advantage is not a consequence of the
implementation's ordering.

\paragraph{Source weighting.}
The sparse weight was swept over $w_s\in\{0.3,0.4,0.5,0.6,0.7\}$ with
$w_d=1-w_s$ (\cref{tab:supp-weight}).

\begin{table}[t]
\centering
\small
\caption{\textbf{Sensitivity to the sparse-source weight on the 156-scene
ScanNet200 evaluation partition.}}
\label{tab:supp-weight}
\setlength{\tabcolsep}{5pt}
\begin{tabular}{l rr rr}
\toprule
& \multicolumn{2}{c}{\textbf{Top-1}} & \multicolumn{2}{c}{\textbf{Full}}\\
\cmidrule(lr){2-3}\cmidrule(lr){4-5}
$w_s$ & Base & Novel & Base & Novel\\
\midrule
0.3 & 45.13 & 20.80 & 46.07 & 27.36\\
0.4 & 45.13 & 20.80 & 45.39 & 28.14\\
0.5 & 45.13 & 20.80 & 44.43 & 28.70\\
0.6 & 43.60 & 19.06 & 43.14 & 28.94\\
0.7 & 43.60 & 19.06 & 41.60 & 28.99\\
\bottomrule
\end{tabular}
\end{table}

Full-distribution fusion remains within 34.17--34.87 HM across the sweep, a
spread of 0.70 points. The Top-1 result is piecewise constant because the
weight acts as a hard switch once disagreeing one-hot vectors are combined:
it can decide which source wins a disagreement but cannot adjust the combined
evidence continuously. Retaining class-level support therefore converts the
source weight from a switch into a graded contribution, and the reported
advantage holds across the full swept range.

\paragraph{Per-source retention.}
The matched ladder applies the same depth to both sources, so it observes
only the diagonal of the (dense, sparse) grid. We therefore evaluate every
combination of Top-1, Top-20, and Full retention in each source, holding the
pool, scenes, and upstream evidence fixed.

\begin{table}[t]
\centering
\small
\caption{\textbf{Per-source retention on the 156-scene ScanNet200
partition.} Entries are HM under the fixed $0.5/0.5$ pool where sparse
evidence exists and dense-only elsewhere. The diagonal reproduces the ladder
of the main paper.}
\label{tab:supp-cross}
\setlength{\tabcolsep}{4pt}
\begin{tabular}{l rrr}
\toprule
& \multicolumn{3}{c}{\textbf{Sparse retention}}\\
\cmidrule(lr){2-4}
\textbf{Dense retention} & Top-1 & Top-20 & Full\\
\midrule
Top-1  & 28.18 & 27.11 & 27.11\\
Top-20 & 21.71 & 35.49 & 35.49\\
Full   & 21.79 & 34.88 & 34.87\\
\bottomrule
\end{tabular}
\end{table}

Retaining alternatives in one source only does not reproduce the gain.
Collapsing the sparse source to Top-1 while retaining dense Top-20 or Full
retention gives 21.71 and 21.79 HM, which is $6.33$ and $6.25$ below the
matched Top-1 control (paired 95\% intervals $[-7.80,-4.87]$ and
$[-7.73,-4.80]$). Collapsing the dense source while retaining the sparse
tail gives 27.11 HM, $1.05$ below the same control ($[-1.87,-0.19]$).
Retaining both tails gives 35.49 HM ($+7.37$ over the control,
$[+6.19,+8.61]$). The sparse axis saturates: its Top-20 and Full cells agree
to within $0.01$ HM, because the sparse posterior covers few classes per
point. The retention effect is therefore a two-source interaction rather
than the property of a single tail.

\paragraph{PoE numerical floor.}
The sparse source contains exact zeros, so the geometric operator uses a
numerical floor $\eps$. Sweeping $\eps\in\{10^{-4},10^{-6},10^{-8},10^{-10},
10^{-12}\}$ gives 36.00, 36.12, 36.16, 36.16, and 36.16 HM respectively. The
total range is 0.16 points and the result is unchanged at the reported
precision for $\eps\le10^{-8}$, so the floor is numerical rather than
tuned.

\paragraph{Cross-dataset balance.}
The ScanNet++ sparse source contributes evidence only for novel categories,
which introduces an explicit base--novel coverage asymmetry. Base mIoU falls
from 32.12 under the dense-only prior to 27.21 under Full retention, while
novel mIoU and HM rise over the same ladder. The ScanNet++ result is therefore
a generalised-balance trade-off rather than a Pareto improvement over the
dense prior on both class groups.

\section{Source-Replacement Diagnostic}
\label{sec:supp-source}

We conduct a five-scene source-dependence diagnostic by replacing SAM3 with a
GroundingDINO \cite{liu2023groundingdino}--SAM2.1 \cite{ravi2024sam2} stack
while keeping the dense RegionPLC source fixed. In this diagnostic, HM
increases monotonically from 14.77 under Top-1 to 15.61 under Top-2, 16.52
under Top-5, 17.14 under Top-10, 18.08 under Top-20, and 18.75 under Full
retention. The paired Top-1-to-Full improvement is $+3.98$ HM
(95\% CI $[+1.11,+6.54]$).

Dense-only RegionPLC reaches 20.70 HM on these five scenes. Consequently, this
diagnostic does not demonstrate that the replacement sparse source produces a
stronger final system, nor does it establish a benchmark-level result. Its
narrower purpose is to show that the directional benefit of preserving
multiple semantic alternatives is also observed with a different sparse
semantic stack. The small number of scenes requires treating this result as
supporting evidence rather than as an independent generalization claim.

\begin{figure*}[t]
\centering
\includegraphics[width=0.96\textwidth]{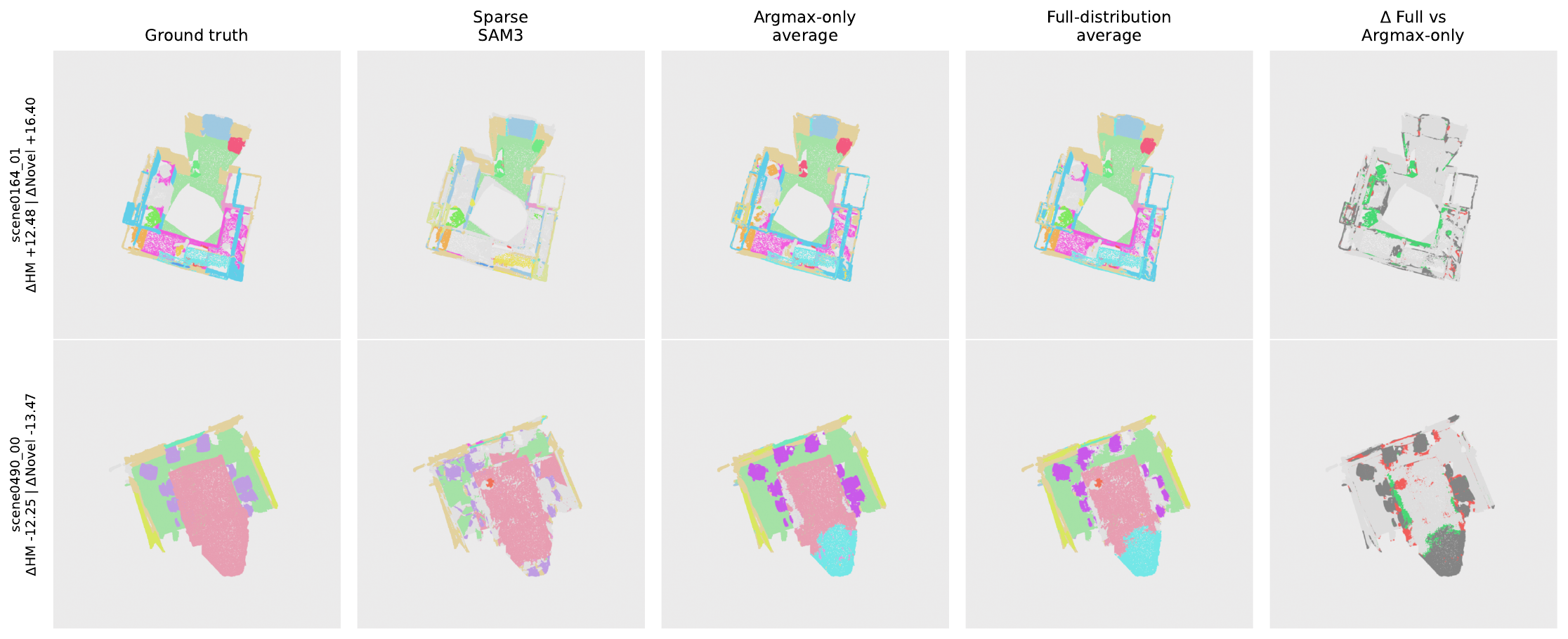}
\caption{\textbf{Point-level qualitative comparison on ScanNet200.}
\emph{Top row (\texttt{scene0164\_01}):} A representative success case where
full-distribution retention preserves secondary hypotheses that downstream
fusion successfully leverages to recover true object semantics ($+12.48$ HM,
$+16.40$ Novel mIoU). In the rightmost difference map, green denotes beneficial
point corrections over hard Top-1 collapse, light grey denotes points correct
under both, and dark grey denotes points wrong under both.
\emph{Bottom row (\texttt{scene0490\_00}):} An instructive failure case where
diffuse secondary semantic evidence associated with systematic confusion
pollutes the fused posterior ($-12.25$ HM, $-13.47$ Novel mIoU; red points
denote newly introduced errors). This failure motivates compact Top-$k$
retention over unconstrained full-distribution tails.}
\label{fig:supp-qualitative}
\end{figure*}

The two rows isolate the two directions of the retention trade-off. In
\texttt{scene0164\_01}, the dense source places the correct class below its
Top-1 decision, and retaining the full distribution allows the sparse source
to recover it. In \texttt{scene0490\_00}, the same mechanism admits weak,
systematically confused mass from the sparse source into the fused posterior.
The difference between the two cases is not the amount of retained mass but
whether the discarded or retained alternatives are supported by the
complementary source.

\section{Failure Cases and Computational Accounting}
\label{sec:supp-failure}

\paragraph{Retention can preserve incorrect alternatives.}
Preserving a distribution does not guarantee that its secondary hypotheses are
useful. In the bottom row of \cref{fig:supp-qualitative}, scene
\texttt{scene0490\_00} decreases from 31.66 HM under the matched Top-1
condition to 19.41 HM under Full-distribution fusion. This corresponds to a
$-12.25$ HM change and a $-13.47$ change in novel mIoU. In this example,
systematic secondary-source confusion introduces misleading support into the
fused posterior.

This failure case complements the aggregate Top-20-versus-Full result in the
main paper, which favors Top-20 by approximately 0.62 HM
(95\% CI $[+0.20,+1.09]$). The appropriate interpretation is therefore not
that all lower-ranked probabilities should be discarded, or that all should be
retained. Rather, premature collapse can remove useful alternatives, while
unrestricted retention can expose fusion to weak or systematically misleading
evidence. The class-level audit reinforces this qualification: Top-20 improves
several large base categories but reduces IoU for a subset of novel
categories.

\paragraph{Measured cache-stage computation.}
The full 312-scene probability cache contains 49,540,568 points. Raw 200-way
dense fp16 distributions occupy 18.46 GiB, while sparse COO evidence occupies
2.93 GiB. The compressed cache footprint, including auxiliary scene data, is
13.31 GiB. The held-out 156 scenes contain 25,020,281 points.

Fusion over the held-out set takes 12.00 seconds on CPU, equivalent to 0.077
seconds per scene or approximately 2.08 million points per second. The
measured process peak resident set size is 719.52 MiB. Loading the compressed
cache takes 130.58 seconds and dominates the cached evaluation wall time.

These measurements characterize post-extraction arbitration only. They exclude
RegionPLC and image-model forward passes, as well as the upstream cost of
producing the sparse image-space predictions. They should therefore not be
interpreted as evidence of end-to-end computational superiority over methods
evaluated with different hardware or different feature-extraction pipelines.

{\small
\setlength{\bibsep}{0.0pt}
\bibliographystyle{ieeenat_fullname}
\bibliography{references}
}
\end{document}